\documentclass{article}

\usepackage[accepted]{icml2026}

\usepackage[utf8]{inputenc}
\usepackage[T1]{fontenc}
\usepackage{hyperref}
\usepackage{url}
\usepackage{booktabs}
\usepackage{amsfonts}
\usepackage{nicefrac}
\usepackage{microtype}
\usepackage{xcolor}
\usepackage{amsmath,amssymb,amsthm}
\usepackage{graphicx}
\usepackage{subcaption}
\usepackage{enumitem}
\usepackage{siunitx}
\newtheorem{proposition}{Proposition}
\newtheorem{remark}{Remark}
\newcommand{\R}{\mathbb{R}}

\icmltitlerunning{Catastrophic Forgetting is Low-Rank}

\begin{document}

\twocolumn[
\icmltitle{Catastrophic Forgetting is Low-Rank: \\
A Function-Space Theory for Continual Adaptation}
\begin{icmlauthorlist}
\icmlauthor{Ido Nitzan Hidekel}{tau}
\icmlauthor{Dan Raviv}{tau}
\end{icmlauthorlist}

\icmlaffiliation{tau}{Tel Aviv University, Tel Aviv, Israel}

\icmlcorrespondingauthor{Ido Nitzan Hidekel}{idon@mail.tau.ac.il}
\icmlcorrespondingauthor{Dan Raviv}{darav@tauex.tau.ac.il}

\icmlkeywords{continual learning, catastrophic forgetting, neural tangent kernel, function space, parameter-efficient fine-tuning}

\vskip 0.3in
]

\printAffiliationsAndNotice{}
\begin{abstract}
Catastrophic forgetting in continual adaptation is usually studied through parameter drift, replay, or distillation, but these views do not identify which output-space directions are vulnerable. We give a function-space account in the NTK regime: new-task training induces old-task prediction drift through the cross-task kernel, yielding a closed-form predictor for the forgetting vector before any new-task gradient step. In frozen-backbone linear-head PEFT-CL, where the model is linear in the trainable parameters, the predictor is exact up to numerical precision; for nonlinear adapters/full fine-tuning it is a local NTK approximation. The same expression reveals that forgetting concentrates in a small number of old-task NTK eigenmodes, and under frozen linear heads gives a Kronecker scaling rule for the vulnerable rank. These results clarify the relation to prior NTK-overlap theory, explain why parameter-space regularizers can miss output-space interference, and motivate a targeted spectral regularizer.
\end{abstract}
\section{Introduction}

Continual adaptation---fine-tuning, alignment maintenance, and adaptation under domain or distribution shift---is bottlenecked by catastrophic forgetting. The problem has gained renewed urgency as adaptation increasingly targets large pretrained models: recent work documents substantial forgetting during continual instruction tuning of LLMs~\citep{luo2024empirical,wang2024llmsurvey}, and a subfield of parameter-efficient continual learning has emerged to adapt frozen backbones without forgetting~\citep{wang2022l2p,wang2022dualprompt,smith2023coda,liang2024inflora}. Despite a decade of mitigation strategies spanning parameter regularization~\citep{kirkpatrick2017overcoming,zenke2017continual}, replay~\citep{rolnick2019experience,buzzega2020dark}, knowledge distillation~\citep{li2017learning}, gradient projection~\citep{farajtabar2020orthogonal,saha2021gradient}, and prompt- or adapter-based PEFT methods, the field still lacks a \emph{mechanistic} account of forgetting. The closest theoretical prior work~\citep{doan2021theoretical,bennani2020generalisation} introduces the cross-task NTK overlap matrix and bounds forgetting magnitude through task alignment, but emphasizes scalar magnitude/risk control rather than the eigenspace structure of the realized forgetting vector: it does not identify which output directions drift, nor why some directions are vulnerable while others are not.

\paragraph{This paper targets the theoretical foundations of continual adaptation.}
We study forgetting as NTK interference in function space. The goal is not to replace PEFT-CL methods, but to identify the output-space directions along which adaptation induces drift. Our main experiments use the analytically exact linear limit: a frozen pretrained backbone with a task-shared trainable linear head. More general frozen-backbone modules such as adapters or LoRA are covered by the same local NTK linearization, but the predictor is then approximate.

\textbf{Contributions.}
\begin{enumerate}[leftmargin=*,topsep=0pt,itemsep=1pt]
\item \textbf{A closed-form forgetting predictor} (Proposition~\ref{prop:predictor}, with full proof) that identifies the forgetting vector direction and magnitude jointly from quantities computable before any new-task gradient step. The predicted shift achieves cosine similarity indistinguishable from $1$ with the realized shift, at float32 precision, on transformer backbones under frozen-backbone linear-head adaptation, and $0.994$ on a frozen ResNet-18 (Section~\ref{sec:theory}).
\item \textbf{A structural characterization of the vulnerable subspace}: in the eigenbasis of the Task-A NTK $K_{AA}$, the forgetting vector concentrates in only a handful of eigenmodes (1--6 modes carry 50--90\% of its energy on Split-MNIST/CIFAR-10). Under a frozen backbone with a trainable linear head of $C$ outputs, $K_{AA}$ admits an exact Kronecker factorization that fixes the vulnerable rank to $k^\star \approx C \cdot k_G$ where $k_G$ is the effective rank of the feature Gram (Remark~\ref{rem:kronecker}).
\item \textbf{A diagnostic lens for method design}: the theory explains why parameter-space methods fail on shared-head benchmarks and why targeted and broad function-space methods converge under the scaling rule. Spectral regularization is offered as an instrument of the theory; its consistent shared-head gains (Table~\ref{tab:shared}) confirm the diagnostic rather than constituting the central claim (Sections~\ref{sec:method},~\ref{sec:experiments}).
\end{enumerate}

\section{Related Work}
\label{sec:related}

\textbf{NTK analyses of continual learning.}\quad \citet{doan2021theoretical} introduce the NTK \emph{overlap matrix} $K_{AB}$ and derive a two-task expression for prediction drift that is closely related to Eq.~\eqref{eq:predictor}; their analysis emphasizes a scalar magnitude bound governed by task alignment. \citet{bennani2020generalisation} derive generalization bounds for SGD and OGD~\citep{farajtabar2020orthogonal} under NTK linearization. Building on this line, our contribution is to read the same two-task expression as a structured operator on output space: we identify the eigenstructure of the old-task kernel that controls which output directions are vulnerable, derive its Kronecker factorization under a frozen linear head, and use the resulting low-rank picture to motivate a targeted spectral regularizer (all formalized in Section~\ref{sec:theory}).

\textbf{PEFT-based continual learning.}\quad Recent pretrained-backbone CL is driven by parameter-efficient adaptation: prompt pools~\citep{wang2022l2p,wang2022dualprompt}, decomposed attention prompting~\citep{smith2023coda}, and interference-motivated low-rank adapters~\citep{liang2024inflora}. These engineer around forgetting without characterizing the interference they mitigate. Our eigenmode concentration and Kronecker factorization are exact under frozen-backbone linear-head adaptation, and provide a diagnostic lens for broader PEFT-CL methods when treated through local NTK linearization.

\textbf{Functional regularization.}\quad Spectral regularization belongs to the functional-regularization family~\citep{benjamin2019measuring,titsias2019functional,li2017learning}: LwF and FRCL distill drift across all output directions, whereas we concentrate the penalty on the NTK-identified vulnerable subspace. Sections~\ref{sec:drift}--\ref{sec:peft_cl} show targeted and broad approaches converge under the scaling rule of Remark~\ref{rem:kronecker}.

\textbf{Parameter-space methods.}\quad GPM~\citep{saha2021gradient}, OGD~\citep{farajtabar2020orthogonal}, EWC~\citep{kirkpatrick2017overcoming}, and SI~\citep{zenke2017continual} constrain parameter drift. Section~\ref{sec:shared_head} shows all four fail on shared-head benchmarks, consistent with our claim that vulnerable directions live in output space, not parameter space.

\textbf{Concurrent mechanistic analyses.}\quad \citet{imanov2026mechanistic} decompose LLM forgetting over architectural components (a \emph{parameter-space} view); ours is a complementary \emph{function-space} decomposition over NTK eigenmodes -- which output directions drift, rather than which components change.

\section{Theory: Forgetting as NTK Interference}
\label{sec:theory}

\begin{figure*}[t]
\centering
\includegraphics[width=0.92\linewidth]{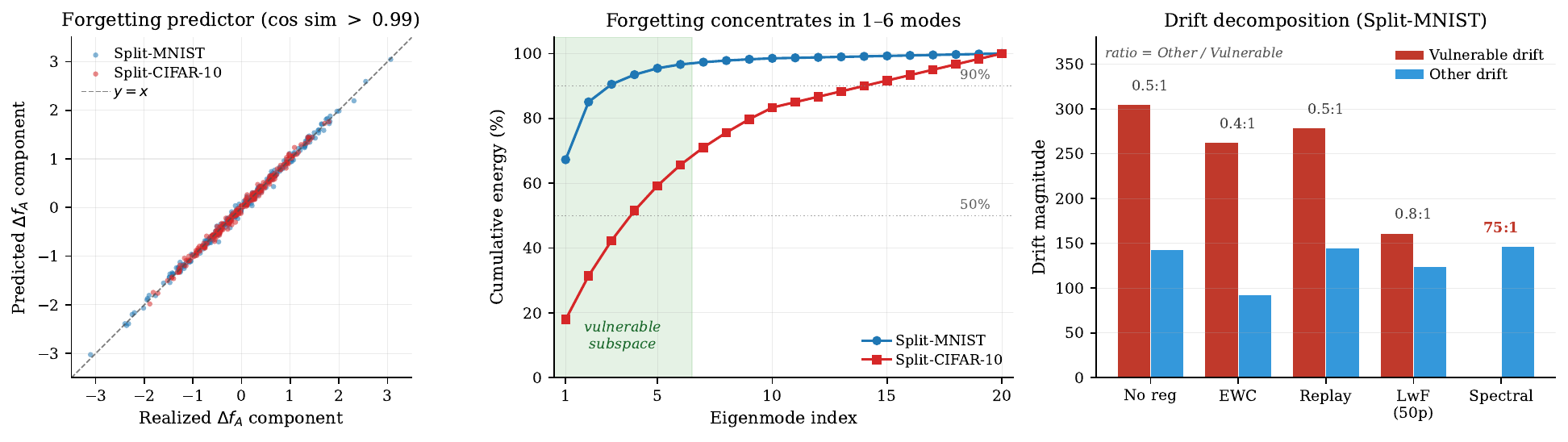}
\caption{\textbf{Left}: Predicted vs.\ realized $\Delta f_A$ (cos sim $>0.99$) on Split-MNIST/CIFAR-10. \textbf{Center}: Cumulative forgetting energy: 50--90\% in 1--6 eigenmodes. \textbf{Right}: Drift decomposition --- spectral reg targets the vulnerable subspace at $75{:}1$ on Split-MNIST ($1.7{:}1$ on the CNN-based Split-CIFAR-10, App.~\ref{app:drift_cifar}) vs.\ $<\!1{:}1$ for baselines.}
\label{fig:main}
\end{figure*}

\subsection{Setup}

We consider two-task continual regression: train on Task~A to $\theta_A \in \R^p$, then on Task~B from $\theta_A$ to $\theta_B$. \emph{Forgetting} is the induced shift in Task-A predictions, written as a flat vector:
\begin{equation}
\Delta f_A \,:=\, \mathrm{vec}\!\left[\,f_A(\theta_B) - f_A(\theta_A)\,\right] \,\in\, \R^{n_A d},
\label{eq:forgetting_def}
\end{equation}
where $\mathrm{vec}[\cdot]$ stacks the $n_A$ rows of the $n_A \times d$ output-shift matrix into a single column. We aim to predict $\Delta f_A$ -- direction and magnitude -- without running Task~B training.

\textbf{Notation.}\quad Let $f(\cdot\,,\theta):\mathcal{X}\to\R^d$ with $\theta\in\R^p$. Given probe set $X_A = \{x_i^A\}_{i=1}^{n_A}$ and training set $(X_B, y_B) = \{(x_i^B, y_i^B)\}_{i=1}^{n_B}$, stack outputs as $f_A(\theta)\in\R^{n_A d}$, $f_B(\theta),\,y_B\in\R^{n_B d}$. All Jacobians and kernels are evaluated at $\theta_A$:
\begin{align*}
J_A &\,:=\, \nabla_\theta f_A(\theta)\big|_{\theta_A}, \quad
J_B \,:=\, \nabla_\theta f_B(\theta)\big|_{\theta_A}, \\
K_{AA} &\,=\, J_A J_A^\top, \quad K_{BB} \,=\, J_B J_B^\top, \quad K_{AB} \,=\, J_A J_B^\top.
\end{align*}
Crucially, $J_B$ uses Task-B inputs but Task-A weights --- this is why the predictor is computable before Task~B training.

\textbf{Task~B objective.}\quad MSE with $L_2$ penalty on drift from $\theta_A$:
\begin{equation}
\mathcal{L}_B(\theta) \,=\, \tfrac{1}{2}\|f_B(\theta) - y_B\|^2 + \tfrac{\lambda}{2}\|\theta - \theta_A\|^2, \quad \lambda \geq 0.
\label{eq:taskB_loss}
\end{equation}
$\lambda = 0$ recovers vanilla MSE (interpreted as the minimum-$\|\delta\|$ interpolant under gradient flow from $\theta_A$), $\lambda > 0$ ensures a unique minimizer. We take $\theta_B \in \arg\min \mathcal{L}_B$.

\subsection{The Predictor}

\begin{proposition}[Forgetting Predictor]
\label{prop:predictor}
Under NTK linearization of $f$ around $\theta_A$~\citep{jacot2018neural} and convergence of Task~B training to a minimizer of $\mathcal{L}_B$,
\begin{equation}
\Delta f_A \,=\, -\,K_{AB}\,(K_{BB} + \lambda I)^{-1}\, r_B,
\label{eq:predictor}
\end{equation}
where $r_B := f_B(\theta_A) - y_B$ is the Task~B residual at $\theta_A$, and the inverse is the Moore-Penrose pseudoinverse when $\lambda = 0$.
\end{proposition}

\begin{proof}[Proof sketch]
With $\delta := \theta - \theta_A$, linearizing $f_B(\theta_A + \delta) \approx f_B(\theta_A) + J_B\delta$ turns $\mathcal{L}_B$ into the convex quadratic $\tfrac{1}{2}\|J_B\delta + r_B\|^2 + \tfrac{\lambda}{2}\|\delta\|^2$, whose minimizer solves the ridge normal equations $(J_B^\top J_B + \lambda I)\delta = -J_B^\top r_B$. The push-through identity $(M^\top M + \lambda I)^{-1} M^\top = M^\top(MM^\top + \lambda I)^{-1}$ with $M = J_B$ rewrites this in the tractable dual form $\delta^\star = -J_B^\top(K_{BB} + \lambda I)^{-1} r_B$. Applying the Task-A linearization, $\Delta f_A \approx J_A\delta^\star = -K_{AB}(K_{BB} + \lambda I)^{-1} r_B$. Full step-by-step derivation in App.~\ref{app:proof}.
\end{proof}

\noindent Empirically, \eqref{eq:predictor} achieves $\cos\mathrm{sim}(\Delta f_A^\text{pred}, \Delta f_A^\text{real}) > 0.99$ on Split-MNIST and Split-CIFAR-10 (Fig.~\ref{fig:main}, left), and is structurally exact in the frozen-backbone linear-head PEFT-CL regime, with $1-\cos\mathrm{sim}$ down to $10^{-6}$ on ViT-B/16 and DINOv2 (App.~\ref{app:precision}).
\subsection{Structural Consequences}
\label{sec:structural}

Proposition~\ref{prop:predictor} delivers $\Delta f_A$ as the action of the cross-task kernel $K_{AB}$ on the Task-B residual. Three structural properties follow directly and shape the rest of the paper: forgetting lives in a low-rank subspace, linearization is exact under a frozen backbone with a trainable linear head, and a Kronecker factorization fixes the vulnerable rank.

\textbf{Low-rank structure.}\quad The predictor in Eq.~\eqref{eq:predictor} factors as
\begin{equation}
\Delta f_A \,=\, -\,J_A J_B^\top (K_{BB} + \lambda I)^{-1}\, r_B,
\label{eq:predictor_factored}
\end{equation}
so $\Delta f_A$ lies in $\mathrm{Im}(J_A) = \mathrm{Im}(K_{AA})$ (for any real $J_A$, $\mathrm{Im}(J_AJ_A^\top)=\mathrm{Im}(J_A)$). Expanding in the eigenbasis $K_{AA} = U\Lambda U^\top$ as $\Delta f_A = \sum_i c_i u_i$ with $c_i = u_i^\top \Delta f_A$, and substituting the SVD $J_A = U \Sigma V_A^\top$ (so $u_i^\top J_A = \sigma_i v_{A,i}^\top$) into Eq.~\eqref{eq:predictor_factored}, the coefficient on mode $i$ is
\begin{equation}
c_i \,=\, -\,\sigma_i\, v_{A,i}^\top J_B^\top (K_{BB} + \lambda I)^{-1}\, r_B,
\label{eq:ci}
\end{equation}
so $|c_i|$ inherits the decay of $\sigma_i$ modulated by the alignment of $v_{A,i}$ with the residual-driven Task-B direction $J_B^\top (K_{BB} + \lambda I)^{-1} r_B$. When $\Lambda$ decays rapidly -- as it does for standard architectures by spectral bias~\citep{rahaman2019spectral} -- and the cross-task alignment factor is not adversarially concentrated on small-$\sigma_i$ directions, $\Delta f_A$ is expected to concentrate in the top eigenmodes of $K_{AA}$. We call $\mathrm{span}(u_1,\ldots,u_k)$ the \emph{vulnerable subspace}: the output directions along which Task-B training can move Task-A predictions. Section~\ref{sec:experiments} measures $k$ empirically and finds 1--6 modes carry 50--90\% of forgetting energy.

\paragraph{Exact linearization under a linear probe.} When $f$ is linear in $\theta$ -- a frozen backbone with a trainable linear head -- the Taylor expansion is an equality and Proposition~\ref{prop:predictor} holds exactly, checkable at machine precision (Section~\ref{sec:experiments}). For nonlinear adapters or full fine-tuning the model is not linear in the trainable parameters, and the predictor becomes a local NTK approximation around $\theta_A$.

\paragraph{Is it forgetting that is low-rank, or learning?} Spectral bias makes the NTK low-rank in any task~\citep{rahaman2019spectral}; the novel content of Eq.~\eqref{eq:ci} is not that $\Lambda$ decays but that the \emph{forgetting} coefficient $c_i$ is set by a specific cross-task product $v_{A,i}^\top J_B^\top (K_{BB}+\lambda I)^{-1} r_B$. Two factors compound: spectral-bias decay of $\sigma_i$ (how Task-A learning is itself low-rank), and a cross-task alignment factor that selects which of those modes Task-B can actually move (App.~\ref{app:lowrank}).

\begin{remark}[Kronecker structure and $k$-scaling]
\label{rem:kronecker}
For a frozen feature map $\phi(x)\in\R^F$ and trainable linear head $W\in\R^{C\times F}$, with $\theta=\mathrm{vec}(W)$ and $f_c(x)=W_c^\top\phi(x)$, the MSE Jacobian is block-diagonal across output classes. Therefore

\begin{equation}
K_{AA} \,=\, I_C \otimes G, \qquad G_{ij} = \phi(x_i)^\top \phi(x_j),
\label{eq:kronecker}
\end{equation}
where $G \in \R^{n_A \times n_A}$ is the (single-output) feature Gram and $I_C$ is the $C \times C$ identity. Consequently every eigenvalue of $G$ has multiplicity $C$ in $K_{AA}$, and the dimension of the vulnerable subspace scales as
\begin{equation}
k^\star \,\approx\, C \cdot k_G,
\label{eq:k_scaling}
\end{equation}
where $k_G$ is the effective rank of $G$ -- the number of dominant eigenvalues of the feature Gram. We observe $k_G \in [1, 5]$ empirically on standard pretrained features (Appendix~\ref{app:sensitivity}). Thus for $C{=}10$ outputs $k^\star \in [10, 50]$, and for $C{=}100$ outputs $k^\star \in [100, 500]$ -- collapsing toward $k^\star \approx 100$ on the near-rank-one frozen CIFAR-100 features ($k_G \approx 1$--$2$ measured; App.~\ref{app:sensitivity}), the matched point used in Section~\ref{sec:peft_cl}. Under softmax cross-entropy, the normalization couples rows of $W$, so Eq.~\eqref{eq:kronecker} holds only approximately and the rule survives as a design heuristic with a softer plateau (Section~\ref{sec:limitations}). This scaling drives the convergence in Section~\ref{sec:peft_cl}: with $C$ large, the LwF-style ``broad'' penalty already lives mostly inside the $C \cdot k_G$-dimensional vulnerable subspace.
\end{remark}

\section{A Theory-Derived Probe: Spectral Regularization}
\label{sec:method}

Since forgetting concentrates in the top-$k$ eigenspace of $K_{AA}$, we penalize drift specifically there. After Task~$\tau$, compute top-$k$ eigenvectors $\{u_j^{(\tau)}\}$ of $K_{\tau\tau}$ on $n_\text{probe}$ probes and store $f_\tau^\text{ref} = f_\tau(\theta_\tau)$, during subsequent training,
\begin{equation}
\mathcal{L}(\theta) = \mathcal{L}_\text{new}(\theta) + \sum_{\tau < t}\frac{\mu}{k}\sum_{j=1}^{k}\left(u_j^{(\tau)\top}\left[f_\tau(\theta) - f_\tau^\text{ref}\right]\right)^2.
\label{eq:spectral_reg}
\end{equation}
Drift in the $(nd{-}k)$-dimensional complement is \emph{unconstrained}, granting full plasticity outside the vulnerable subspace -- unlike EWC (all $p$ parameter directions) and LwF (all $nd$ output directions).

\section{Experiments}
\label{sec:experiments}

\textbf{Setup.}\quad Three benchmarks: Split-MNIST (5 tasks, MLP) and Split-CIFAR-10 (5 tasks, $\sim$200k-param CNN), each in shared-head (single output layer over the union of classes -- the harder regime) and multi-head variants; Split-CIFAR-100 with a frozen ImageNet ResNet-18 (10 tasks, frozen-backbone linear-head PEFT-CL regime, where linearization is exact). Baselines span parameter regularization (EWC~\citep{kirkpatrick2017overcoming}, SI~\citep{zenke2017continual}), gradient projection (GPM~\citep{saha2021gradient}), functional distillation (LwF~\citep{li2017learning}), replay (random, DER++~\citep{buzzega2020dark}), and a no-reg control. We report final-task accuracy and \emph{forgetting} (mean per-task accuracy drop) over 5--10 seeds. Apps.~\ref{app:multihead}--\ref{app:sensitivity} cover multi-head, probe scaling, and sensitivity.

\begin{table}[t!]
\centering
\caption{\textbf{Shared-head results:} parameter-space methods fail, function-space methods succeed. Std shown for the top two methods (5/10 seeds). Spectral beats LwF on CIFAR-10 ($p{=}0.002$).}
\label{tab:shared}
\small
\setlength{\tabcolsep}{5pt}
\renewcommand{\arraystretch}{0.95}
\begin{tabular}{lcccc}
\toprule
& \multicolumn{2}{c}{Split-MNIST} & \multicolumn{2}{c}{Split-CIFAR-10} \\
\cmidrule(lr){2-3}\cmidrule(lr){4-5}
Method & Acc $\uparrow$ & Fgt $\downarrow$ & Acc $\uparrow$ & Fgt $\downarrow$ \\
\midrule
No reg    & 19.7 & 99.6 & 17.5 & 85.8 \\
EWC       & 19.8 & 99.5 & 18.7 & 88.7 \\
SI        & 22.2 & 96.4 & 17.8 & 85.5 \\
\midrule
Replay    & 56.7 & 53.3 & 21.5 & 82.7 \\
DER++     & 65.9 & 42.0 & 28.4 & 78.6 \\
\midrule
LwF       & \textbf{80.3}\tiny{$\pm$1.4} & 23.9\tiny{$\pm$1.8} & 29.0\tiny{$\pm$2.0} & 77.4\tiny{$\pm$2.4} \\
Spectral  & 77.9\tiny{$\pm$3.0} & \textbf{11.3}\tiny{$\pm$0.5} & \textbf{32.0}\tiny{$\pm$2.1} & \textbf{51.6}\tiny{$\pm$5.1} \\
\bottomrule
\end{tabular}
\end{table}
\subsection{Parameter-Space Methods Fail on Shared-Head}
\label{sec:shared_head}
Table~\ref{tab:shared} shows a clean divide: parameter-space methods (EWC, SI) match no-reg while function-space methods improve substantially. The mechanism is geometric -- forgetting concentrates in a low-dimensional subspace of \emph{output} space, but diagonal Fisher is anisotropic in parameter coordinates and is not constructed to align with the rank-$k$ output-space projector $U_k U_k^\top$, so EWC can slow drift broadly without preferentially protecting the vulnerable $K_{AA}$ eigenmodes. A single-step drift-decomposition probe confirms this: EWC and no-reg leave indistinguishable Other:Vuln ratios ($0.033{:}1$ vs $0.028{:}1$), whereas spectral reg flips it to $32.7{:}1$; EWC scales both components down uniformly but cannot translate per-parameter Fisher mass into selective output-mode protection (App.~\ref{app:ewc}). Non-diagonal methods (GPM) fail for the same reason (App.~\ref{app:gpm}).

\subsection{Drift Decomposition: Targeted Protection}
\label{sec:drift}
\begin{table}

\centering
\caption{\textbf{Direct test of the low-rank claim.} Drift decomposition after 5 tasks on Split-MNIST (shared-head, $k{=}10$). Spectral reg suppresses vulnerable-subspace drift $150\times$. CIFAR-10 in App.~\ref{app:drift_cifar}.}
\label{tab:drift}
\small
\begin{tabular}{lccc}
\toprule
Method & Vuln.\ $\downarrow$ & Other & Ratio \\
\midrule
No reg              & 305.1 & 142.8 & 0.5:1 \\
EWC (best)          & 262.5 & 92.9  & 0.4:1 \\
Replay (100)        & 278.9 & 144.5 & 0.5:1 \\
LwF (50p)           & 161.4 & 124.4 & 0.8:1 \\
Spectral ($\mu{=}10$) & \textbf{2.0} & 146.7 & \textbf{75:1} \\
\bottomrule
\end{tabular}
\end{table}
Decomposing $\Delta f_A$ inside vs.\ outside the vulnerable subspace isolates the targeting mechanism (Table~\ref{tab:drift}): LwF reduces both proportionally, while spectral reg suppresses only the vulnerable component -- exactly the structural distinction the theory predicts.
\subsection{Function-Space Methods Converge Under the Scaling Rule}
\label{sec:peft_cl}
On Split-CIFAR-100 ($C{=}100$, 20 probes), spectral reg at $k{=}100$ and LwF are indistinguishable ($39.9{\pm}1.5\%$ vs $39.8{\pm}1.2\%$): matched at $k^\star \approx C \cdot k_G$, targeted and broad function-space methods converge. Outside this regime they trade off as predicted (App.~\ref{app:probe}), and spectral protection stabilizes the old-class decision boundary against inter-task \emph{confusion} (App.~\ref{app:cil}).
\section{Conclusion}
This work reframes catastrophic forgetting as low-rank interference in function space. In the NTK regime, a cross-task kernel expression predicts the forgetting vector before new-task training and identifies a small vulnerable eigenspace of $K_{AA}$. For frozen-backbone linear heads, the analysis is exact and yields $k^\star \approx Ck_G$; for nonlinear adapters or full fine-tuning, it is a local approximation. Empirically, this explains why parameter-space regularizers can fail on shared-head benchmarks and why spectral function-space regularization protects the relevant modes. Extending the mechanism to cross-entropy, evolving representations, and long task sequences remains open.

\small
\bibliographystyle{plainnat}

\normalsize
\newpage
\appendix
\section{Proof of Proposition~\ref{prop:predictor}}
\label{app:proof}

Let $\delta := \theta - \theta_A$.

\textbf{Step 1: Linearize around $\theta_A$.}
\begin{equation}
f_B(\theta_A + \delta) \,\approx\, f_B(\theta_A) + J_B \delta. \label{eq:linearize}
\end{equation}

\textbf{Step 2: Reduce to quadratic.} Substituting \eqref{eq:linearize} into \eqref{eq:taskB_loss}:
\begin{equation}
\mathcal{L}_B(\theta_A + \delta) \,\approx\, \tfrac{1}{2}\|J_B \delta + r_B\|^2 + \tfrac{\lambda}{2}\|\delta\|^2. \label{eq:quadratic_loss}
\end{equation}
Convex in $\delta$, with a unique minimizer (minimum-norm for $\lambda=0$).

\textbf{Step 3: First-order condition.} Setting $\nabla_\delta \mathcal{L}_B = 0$ yields the ridge normal equations
$(J_B^\top J_B + \lambda I_p)\,\delta = -J_B^\top r_B$, hence
\begin{equation}
\delta^\star \,=\, -(J_B^\top J_B + \lambda I_p)^{-1} J_B^\top\, r_B. \label{eq:delta_primal}
\end{equation}

\textbf{Step 4: Push-through identity.} For any $M \in \R^{m\times p}$,
\begin{equation}
(M^\top M + \lambda I_p)^{-1} M^\top \,=\, M^\top (MM^\top + \lambda I_m)^{-1},
\label{eq:pushthrough}
\end{equation}
applied to \eqref{eq:delta_primal} with $M = J_B$:
\begin{equation}
\delta^\star \,=\, -J_B^\top\,(K_{BB} + \lambda I)^{-1}\,r_B. \label{eq:delta_dual}
\end{equation}
The inverse is now $n_B d \times n_B d$ rather than $p \times p$ -- the tractable dual form.

\textbf{Conclusion.} Applying the linearization to Task~A:
\begin{align*}
\Delta f_A \,&\approx\, J_A\, \delta^\star
\,=\, -J_A J_B^\top\,(K_{BB} + \lambda I)^{-1}\,r_B \\
\,&=\, -K_{AB}\,(K_{BB} + \lambda I)^{-1}\,r_B. \qquad\blacksquare
\end{align*}

\section{Limitations}
\label{sec:limitations}

\textbf{Scope: where the predictor is exact and where it is approximate.}\quad Proposition~\ref{prop:predictor} is derived under four assumptions: (i) NTK linearization of $f$ around $\theta_A$, (ii) convergence of Task-B training to a minimizer of $\mathcal{L}_B$, (iii) MSE loss with an optional ridge anchor, and (iv) trainable parameters restricted to a head (or low-rank adapter) on top of frozen pretrained features (the frozen-backbone linear-head PEFT-CL regime). When all four hold simultaneously -- as in frozen-backbone linear probing -- the predictor is an exact identity at machine precision (Table~\ref{tab:multi_backbone}): it is exact precisely when the model is linear in the trainable parameters (a trainable linear head on a frozen backbone), and is a local NTK approximation around $\theta_A$ for nonlinear adapters or full fine-tuning. When the head is trainable but the backbone is frozen and shallow MLPs/CNNs are trained from scratch, the predictor remains tight ($\cos\mathrm{sim} > 0.99$) because the operating regime is close to lazy. Outside these regimes the predictor degrades along three orthogonal axes, which we list explicitly because they delimit the claims of this paper:
\begin{itemize}[leftmargin=*,topsep=2pt,itemsep=1pt]
\item \emph{Full fine-tuning of deep backbones.} The Jacobian $J_A$ evaluated at $\theta_A$ does not see feature drift in lower layers, and $K_{AA}$ itself becomes time-varying during Task-B training with eigenstructure that evolves as representations do. The predictor then captures only the linearized component of $\Delta f_A$.
\item \emph{Continual instruction tuning of LLMs.} The frozen-backbone linear-head PEFT-CL regime assumption (small adapter/head) is closer to LoRA-style instruction tuning than to full fine-tuning, but the MSE assumption is violated: token-level cross-entropy with a softmax over a large vocabulary breaks both the Kronecker factorization of Remark~\ref{rem:kronecker} and the closed-form residual $r_B$.
\item \emph{Long task sequences with cumulative drift.} Even when each pairwise step is well predicted, error in $K_{AA}$ compounds across tasks; we measure this drift in App.~\ref{app:stability} but do not give a multi-step predictor.
\end{itemize}
Extension via gradient sketching, block-diagonal Jacobian approximations, an online estimate of $K_{AA}$ tracked through training, and a softmax-aware analogue of Eq.~\eqref{eq:predictor} are natural next steps.

\textbf{Cross-entropy breaks strict block-decoupling.}\quad Remark~\ref{rem:kronecker}'s factorization $K_{AA} = I_C \otimes G$ relies on MSE loss, under which output rows are gradient-decoupled. Softmax cross-entropy couples rows through the normalization, so the factorization holds only approximately. The $k^\star \approx C \cdot k_G$ rule remains a useful design heuristic under CE but may require modest empirical adjustment consistent with the soft plateau we observe on MNIST at $k \in [10, 50]$ rather than a sharp optimum (App.~\ref{app:sensitivity}).

\textbf{Relation to prior NTK-CL theory.}\quad The two-task expression in Proposition~\ref{prop:predictor} is closely related to the form derived by~\citet{doan2021theoretical}: both arise from substituting the NTK linearization into a quadratic loss and applying the push-through identity. Our contribution is not to add a new derivation but to read this expression as a structured operator on output space, exposing the eigenstructure of $K_{AA}$, the Kronecker factorization in the frozen-backbone linear-head regime, and the resulting scaling rule that ties together targeted and broad function-space methods. The closed-form predictor and the structural-rank picture should be taken as a single package.

\section{Predictor Precision Across Backbones}
\label{app:precision}

\begin{table}[h]
\centering
\caption{Predictor precision across frozen backbones (Split-CIFAR-100, MSE, 3 seeds). Values are $1 - \cos\mathrm{sim}$, lower is sharper. ResNet-18's gap is an SGD-convergence artifact on a less well-conditioned feature Gram.}
\label{tab:multi_backbone}
\small
\setlength{\tabcolsep}{4pt}
\begin{tabular}{lcc}
\toprule
Backbone & 10 tasks & 20 tasks \\
\midrule
ResNet-18  & $6.4 \pm 1.2\,(\text{e-}3)$ & $6.1 \pm 1.4\,(\text{e-}3)$ \\
ViT-B/16   & $6 \pm 4\,(\text{e-}6)$     & $35 \pm 35\,(\text{e-}6)$ \\
DINOv2     & $<1\,(\text{e-}6)$          & $<1\,(\text{e-}6)$ \\
\bottomrule
\end{tabular}
\end{table}

In the frozen-backbone linear-head PEFT-CL regime the NTK linearization is exact, so Proposition~\ref{prop:predictor} holds at machine precision; the residual gap on ResNet-18 reflects SGD convergence rather than linearization error.

\section{Method Overview Figure}
\label{app:overview}

\begin{figure}[H]
\centering
\includegraphics[width=\linewidth]{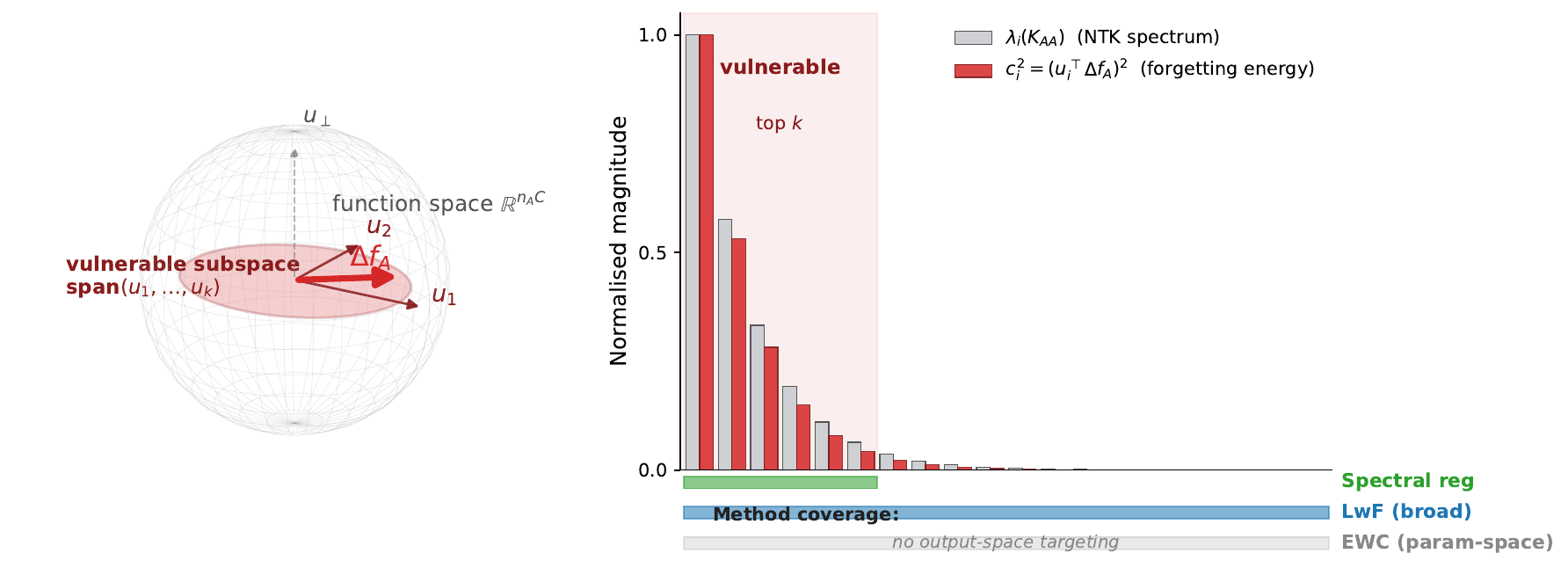}
\caption{\textbf{Method overview.} \emph{Left}: Forgetting $\Delta f_A$ lies in the column space of $K_{AA}$ (Prop.~\ref{prop:predictor}), and its energy concentrates on a low-rank slice -- the vulnerable subspace $\mathrm{span}(u_1,\ldots,u_k)$ spanned by the top eigenvectors of $K_{AA}$. The complementary $u_\perp$ directions are unprotected by construction. \emph{Right}: the NTK spectrum decays rapidly (grey), and the forgetting-energy coefficients $c_i^2 = (u_i^\top \Delta f_A)^2$ inherit this decay (red); 50--90\% of $\Delta f_A$'s energy lives in the first 1--6 modes (Fig.~\ref{fig:main}, center). \emph{Spectral reg} (\S\ref{sec:method}) penalizes drift exactly on those modes and leaves the complement free; \emph{LwF} penalizes every output direction uniformly; \emph{EWC} acts in parameter space and has no mechanism to selectively reach output modes (\S\ref{sec:shared_head}, App.~\ref{app:ewc}).}
\label{fig:method}
\end{figure}

\section{Low-Rank Forgetting vs.\ Low-Rank Learning}
\label{app:lowrank}

\textbf{The confound.} The center panel of Fig.~\ref{fig:main} shows the forgetting vector concentrating in the top eigenmodes of $K_{AA}$. Taken alone this is not conclusive: the NTK spectrum $\lambda_i$ already decays rapidly by spectral bias~\citep{rahaman2019spectral}, so \emph{any} vector expressed in this eigenbasis tends to look top-heavy. Concentration could therefore be a generic property of the architecture rather than a fact about forgetting. To see which, write the per-mode coefficient of Eq.~\eqref{eq:ci} as a product of two factors:
\begin{equation}
c_i \;=\; \underbrace{\sigma_i}_{\text{(i) spectral decay}}\;\cdot\;\underbrace{v_{A,i}^\top J_B^\top (K_{BB}+\lambda I)^{-1} r_B}_{\text{(ii) cross-task alignment}}.
\label{eq:two_factors}
\end{equation}
Factor (i) is the singular value $\sigma_i=\sqrt{\lambda_i}$; it shrinks high modes \emph{regardless of the task} -- pure spectral bias, shared by learning and forgetting. Factor (ii) depends on the \emph{actual} new task: it measures how the Task-B residual $r_B$, passed through the new-task kernel, projects onto the $i$-th old-task direction. The eigenvalue decay (i) is not the novel content; the task-specific selection (ii) is.

\textbf{A control that isolates the two factors.} To test whether the observed concentration is (i) or (ii), we keep the \emph{same} operators $K_{AB},K_{BB}$ (so factor (i) is untouched) but replace the real residual $r_B$ in Eq.~\eqref{eq:predictor} with a random residual of equal norm, which randomizes factor (ii). If the concentration were just spectral bias, the random residual -- fed through the identical decaying kernel -- would concentrate just as tightly. On Split-MNIST (3 seeds) we report, for four energy profiles in the relevant NTK eigenbasis, the participation ratio $\mathrm{PR}=(\sum_i e_i)^2/\sum_i e_i^2$ (lower $=$ more concentrated) and the number of modes carrying 50/90\% of the energy ($k_{50}/k_{90}$), in an $n_AC{=}500$-dimensional output space:

\begin{center}
\small
\begin{tabular}{lcccc}
\toprule
Energy profile & $k_{50}$ & $k_{90}$ & top-1 & PR \\
\midrule
$K_{AA}$ spectrum $\lambda_i$ (capacity)  & 6.0 & 47.3 & 0.15 & 16.5 \\
Realized forgetting $\Delta f_A$          & 1.0 & 4.0  & 0.78 & \textbf{1.6} \\
Random-residual control                   & 52.0 & 217.0 & 0.01 & 123.0 \\
Realized learning $\Delta f_B$            & 1.3 & 3.7  & 0.65 & 2.2 \\
\bottomrule
\end{tabular}
\end{center}
The realized forgetting vector is far more concentrated ($\mathrm{PR}{=}1.6$, $97\%$ of its energy in $6$ modes, $78\%$ in a single mode) than the bare $K_{AA}$ spectrum ($\mathrm{PR}{=}16.5$): spectral bias alone does \emph{not} explain it. The random-residual control, fed through the identical operators, instead scatters across $\mathrm{PR}{=}123$ modes -- \emph{more} spread than even the bare spectrum. So the tight concentration is produced by factor (ii), the cross-task alignment of the \emph{real} residual, not by the kernel's decay. This is precisely why a cross-task predictor is needed and the spectrum is not enough: factor (i) says the low-eigenvalue modes \emph{exist}, but only factor (ii) -- the $K_{AB}$/$r_B$ product -- says \emph{which} of them the new task will actually disturb.

This also settles the question of whether it is forgetting or learning that is low-rank. The learning drift $\Delta f_B$, decomposed in the $K_{BB}$ eigenbasis, is itself low-rank ($\mathrm{PR}{=}2.2$) -- a direct consequence of spectral bias, factor (i). Both learning and forgetting are low-rank for that shared reason; forgetting is \emph{additionally} sharpened by cross-task alignment (factor (ii)) into a still lower-dimensional, task-pair-specific subspace. The novelty of the low-rank claim is thus not the decay of $\Lambda$ (generic) but the alignment-driven selection of a particular vulnerable direction.

\section{Why EWC Cannot Protect the Vulnerable Subspace}
\label{app:ewc}

EWC~\citep{kirkpatrick2017overcoming} penalizes parameter drift weighted by the diagonal Fisher information $F_{ii}$, anchoring each parameter to its Task-A value with strength proportional to $F_{ii}$. The vulnerable subspace, by contrast, is defined in \emph{output} space: it is the span of the top-$k$ eigenvectors of $K_{AA} = J_A J_A^\top$. There is no general mechanism by which a per-parameter diagonal penalty maps onto a selected set of output-space eigenmodes. Diagonal Fisher is anisotropic in parameter coordinates, but it is not constructed to align with the rank-$k$ output-space projector $U_k U_k^\top$. Thus EWC can slow drift broadly, but it need not preferentially protect the vulnerable $K_{AA}$ eigenmodes.

We verify this with a single-step drift-decomposition experiment that asks exactly the question of interest: does EWC's regularization preferentially suppress drift inside the top-$k$ eigenspace of $K_{AA}$? After training Task~A, we record $U_k$ on probes, then train Task~B under (i) no regularization, (ii) EWC, and (iii) spectral regularization, and decompose the realized $\Delta f_A$ into its vulnerable and complementary components (Split-MNIST, $k{=}10$, 3 seeds):

\begin{center}
\small
\begin{tabular}{lccc}
\toprule
Method & $\|U_k^\top \Delta f_A\|^2$ & $\|(I{-}U_kU_k^\top)\Delta f_A\|^2$ & Other:Vuln \\
\midrule
No reg   & $177{\rm k} \pm 20{\rm k}$ & $5.0{\rm k}\pm 0.1{\rm k}$ & $0.028\!:\!1$ \\
EWC      & $\;85{\rm k} \pm 9{\rm k}$ & $2.8{\rm k}\pm 0.5{\rm k}$ & $0.033\!:\!1$ \\
Spectral & $\;\;\;25 \pm 6$            & $820 \pm 255$              & $32.7\!:\!1$ \\
\bottomrule
\end{tabular}
\end{center}

EWC and no-reg have indistinguishable Other:Vuln ratios: EWC scales both components down by roughly $2\times$, but does not preferentially shrink the vulnerable component. Spectral regularization flips the ratio by three orders of magnitude. This is the operational answer to how Fisher-based regularization aligns with the vulnerable subspace: under realized drift, it does not. The five-task drift decomposition in Section~\ref{sec:drift} confirms the same pattern at scale (Table~\ref{tab:drift}).

\textbf{The failure is not mis-weighting -- it is the parameter-space penalty form.} A reviewer might expect EWC to fail because diagonal Fisher \emph{mis-points}, putting its mass on the wrong (non-vulnerable) parameters. The opposite is true. Pushing the diagonal Fisher $F_{ii}$ through $J_A$ into output space and comparing it to the per-parameter vulnerable-mode mass $\mathrm{diag}(J_A^\top U_kU_k^\top J_A)$ (Split-MNIST, $k{=}10$, 3 seeds), the two are \emph{strongly} aligned: Spearman $\rho = 0.996$, cosine $0.94$, and $55\%$ of the total Fisher trace already lies inside the top-$k$ vulnerable subspace. EWC's weights sit on essentially the right parameters. Yet realized drift (above) shows no selective protection, because a diagonal \emph{quadratic} parameter anchor only rescales how fast each coordinate moves; it cannot impose the rank-$k$ output-space constraint $U_k^\top\Delta f_A \approx 0$ that the geometry requires. Selective protection of a low-rank \emph{output} subspace needs a function-space penalty (Eq.~\eqref{eq:spectral_reg}), not a better-weighted parameter-space one -- which is precisely why spectral regularization succeeds where EWC, with near-identical Fisher targeting, does not.

\section{Inter-Task Confusion in Class-Incremental Learning}
\label{app:cil}

In class-incremental settings, the shared-head classifier is never jointly trained to discriminate between classes from different tasks; the dominant failure mode is inter-task \emph{confusion} -- the relative ordering of old-task logits is corrupted by drift introduced when fitting new-task logits. Our predictor identifies precisely this corruption in closed form: the top eigenmodes of $K_{AA}$ that carry $\Delta f_A$ are the directions along which the old-class logit field can move. Protecting them with spectral regularization stabilizes the old-class decision boundary against the cross-task force $K_{AB}(K_{BB}+\lambda I)^{-1} r_B$, while leaving the head free to fit new-task logits in the complement. This is a structural reason for the consistent gains seen on shared-head benchmarks (Table~\ref{tab:shared}), and explains why parameter-space methods, which lack a way to selectively constrain the old-class output field, cannot recover inter-task discrimination.

We make this concrete on Split-MNIST: after training the new task on a shared $10$-way head, we evaluate on held-out \emph{old-task} test images (true labels $0$--$4$) and measure the inter-task confusion rate -- the fraction predicted into a new-task class $\{5,\ldots,9\}$ -- alongside accuracy restricted to the old-class logits (task identity known), 3 seeds:
\begin{center}
\small
\begin{tabular}{lccc}
\toprule
Method & Old-class acc.\ $\uparrow$ & Confusion $\to$ new $\downarrow$ & Acc.\ $|$ task $\uparrow$ \\
\midrule
No reg   & $0.0\%$  & $100.0\%$ & $29.8\%$ \\
EWC      & $1.3\%$  & $98.7\%$  & $84.4\%$ \\
Spectral & $\mathbf{93.5\%}$ & $\mathbf{0.0\%}$ & $\mathbf{93.5\%}$ \\
\bottomrule
\end{tabular}
\end{center}
The decomposition is revealing: EWC actually preserves the \emph{within}-old-task ordering ($84.4\%$ accuracy once the task is known), yet $98.7\%$ of old-class images are still classified as new-task classes -- the new logits overwhelm the old ones because EWC cannot constrain the cross-task output direction. Spectral regularization, by protecting exactly that direction, drives inter-task confusion from $100\%$ to $0\%$ and recovers old-class accuracy without any task oracle.

\section{GPM and Other Non-Diagonal Parameter-Space Methods}
\label{app:gpm}

GPM~\citep{saha2021gradient} extends parameter-space regularization beyond EWC's diagonal Fisher by projecting new-task gradients orthogonal to the top-energy subspace of old-task gradients. It is therefore not subject to the rank mismatch of Section~\ref{sec:shared_head}: its protection operator is itself low-rank. The mechanism it targets, however, is different from ours. GPM identifies directions in \emph{parameter} space along which old-task gradients had large energy, whereas the vulnerable subspace identified by Proposition~\ref{prop:predictor} consists of directions in \emph{output} space along which $K_{AA}$ has large eigenvalues. On Split-CIFAR-100 (10 tasks, frozen ResNet-18, frozen-backbone linear-head PEFT-CL regime), GPM achieves $9.7$--$10.1\%$ final-task accuracy across energy thresholds, matching diagonal-Fisher EWC ($9.9\%$) and dominated by spectral regularization ($39.9\%$) and LwF ($39.8\%$). Gradient-energy bases simply do not align with output-interference directions -- which function-space methods target by construction. The conclusion is that the failure of parameter-space methods on shared-head benchmarks is not specific to diagonal Fisher: it is a consequence of solving the wrong geometric problem (input/parameter gradient energy) for the wrong target (output-space drift).

\section{Subspace Stability}
\label{app:stability}

We measure whether the vulnerable subspace $\mathrm{span}(u_1, \ldots, u_k)$ remains valid as the model trains on subsequent tasks, via principal angles between $U_k$ at $\theta_A$ (after Task~0) and $U_k$ recomputed at each subsequent task boundary on the \emph{same} Task-0 probes. On Split-MNIST, the mean principal angle plateaus at $25.9^\circ$ after 4 task transitions; the bulk geometry is preserved, though individual directions may rotate up to $83^\circ$. On Split-CIFAR-10, rotation is faster (mean $34.0^\circ$), quantitatively explaining why spectral regularization gains are smaller on CNNs than MLPs. The gap between mean (${\sim}25$-$35^\circ$) and max (${\sim}80$-$90^\circ$) angles reveals anisotropic rotation: a few eigendirections rotate substantially while the majority remain stable.

\section{Cross-Task Coupling Predicts Forgetting Magnitude}
\label{app:coupling}

The predictor identifies the \emph{direction} of forgetting; the Frobenius norm of $K_{AB}$ also enables \emph{magnitude forecasting} across task pairs. Across all 10 task pairs on Split-MNIST (3 seeds, 30 measurements), $\|K_{AB}\|_F$ strongly predicts realized forgetting magnitude (Spearman $\rho = 0.88$, $p < 10^{-10}$): before training on any new task, one can cheaply estimate which prior tasks will suffer the most damage. On CIFAR-10 the correlation is weaker ($\rho = 0.36$, $p = 0.053$), consistent with stronger NTK regime violations on CNNs.

\section{Multi-Head Evaluation}
\label{app:multihead}
Task-specific heads eliminate the cross-output interference that makes shared-head catastrophic. EWC recovers from $19.8\%$ (shared) to $88.4\%$ (multi) on CIFAR, confirming that its shared-head failure arises from output-layer interference, not a fundamental flaw in parameter-space regularization. Spectral reg is competitive on MNIST but no longer leads on multi-head CIFAR --- exactly what Remark~\ref{rem:kronecker} predicts: the $C$-fold eigenvalue multiplicity does not apply per-head, so the scaling-rule advantage that drives shared-head dominance disappears. The result confirms the theory's scope rather than contradicting it.

\begin{table}[h!]
\centering
\caption{Multi-head evaluation. Task-specific heads eliminate shared-head output interference: EWC recovers on CIFAR, and LwF leads.}
\label{tab:multi_app}
\small
\setlength{\tabcolsep}{5pt}
\renewcommand{\arraystretch}{0.95}
\begin{tabular}{lcccc}
\toprule
& \multicolumn{2}{c}{MNIST Multi-Head} & \multicolumn{2}{c}{CIFAR Multi-Head} \\
\cmidrule(lr){2-3}\cmidrule(lr){4-5}
Method & Acc (\%) $\uparrow$ & Fgt (\%) $\downarrow$ & Acc (\%) $\uparrow$ & Fgt (\%) $\downarrow$ \\
\midrule
No reg              & 92.2\tiny{$\pm$2.3} & 9.4\tiny{$\pm$2.8} & 79.2\tiny{$\pm$2.2} & 16.7\tiny{$\pm$2.7} \\
SI                  & 97.3\tiny{$\pm$0.8} & 3.0\tiny{$\pm$0.9} & 83.1\tiny{$\pm$4.0} & 11.1\tiny{$\pm$4.8} \\
Spectral $k{=}10$   & 99.1\tiny{$\pm$0.2} & 0.5\tiny{$\pm$0.2} & 81.1\tiny{$\pm$1.1} & 10.4\tiny{$\pm$0.9} \\
Spectral $k{=}50$   & \textbf{99.4}\tiny{$\pm$0.1} & \textbf{0.4}\tiny{$\pm$0.1} & 86.3\tiny{$\pm$1.5} & 7.0\tiny{$\pm$1.9} \\
EWC                 & 98.4\tiny{$\pm$0.7} & 1.7\tiny{$\pm$0.9} & 88.4\tiny{$\pm$0.6} & \textbf{1.2}\tiny{$\pm$0.4} \\
LwF (50p)           & 99.3\tiny{$\pm$0.1} & 0.5\tiny{$\pm$0.1} & \textbf{89.2}\tiny{$\pm$1.2} & 4.6\tiny{$\pm$1.0} \\
\bottomrule
\end{tabular}
\end{table}
\section{CIFAR-10 Drift Decomposition}
\label{app:drift_cifar}

\begin{table}[h]
\centering
\caption{Drift decomposition after 5 tasks on Split-CIFAR-10 (shared-head, $k{=}10$).}
\small
\begin{tabular}{lccc}
\toprule
Method & Vuln.\ $\downarrow$ & Other & Ratio \\
\midrule
No reg              & 146.6 & 106.1 & 0.7:1 \\
EWC (best)          & 425.4 & 263.3 & 0.6:1 \\
Replay (200)        & 302.1 & 203.7 & 0.7:1 \\
LwF (50p)           & 139.5 & 76.6  & 0.5:1 \\
Spectral ($\mu{=}10$) & \textbf{30.8} & 52.6 & \textbf{1.7:1} \\
\bottomrule
\end{tabular}
\end{table}

Targeting is weaker on CIFAR-10 (ratio $1.7$:$1$) than MNIST ($75$:$1$), consistent with the faster subspace rotation on CNNs (Appendix~\ref{app:stability}). Spectral reg remains the only method where other drift exceeds vulnerable drift.

\section{Probe Scaling}
\label{app:probe}

\begin{table}[h]
\centering
\caption{Probe scaling on Split-MNIST. Spectral reg extracts more anti-forgetting signal per probe at $\leq$100 probes, LwF overtakes at 200.}
\small
\begin{tabular}{lcc}
\toprule
Probes & Spectral $k{=}50$ & LwF \\
\midrule
20  & \textbf{71.8}\tiny{$\pm$2.4} & 67.1 \\
50  & \textbf{84.6}\tiny{$\pm$1.2} & 80.3\tiny{$\pm$1.4} \\
100 & \textbf{87.9}\tiny{$\pm$0.7} & 87.2 \\
200 & 88.3\tiny{$\pm$0.4} & \textbf{92.4} \\
\bottomrule
\end{tabular}
\end{table}

At low probe count, KL divergence is diluted across all output dimensions each direction receives a weak supervisory signal. Eigenmode projection concentrates the penalty on $k$ directions, extracting more anti-forgetting value per stored point. At high probe count, LwF's broad coverage pays off: per-direction signal becomes strong across all dimensions. Targeted regularization wins at low memory, broad regularization wins at high memory.

\section{Sensitivity Analysis}
\label{app:sensitivity}

\begin{table}[h]
\centering
\caption{Sensitivity to $k$ (Split-MNIST, $\mu{=}1$, $C{=}10$). The plateau at $k \in [10, 50]$ matches Remark~\ref{rem:kronecker}: $k^\star \approx C \cdot k_G$ with $k_G \in [1, 5]$ predicts $k^\star \in [10, 50]$.}
\small
\setlength{\tabcolsep}{4pt}
\resizebox{\columnwidth}{!}{%
\begin{tabular}{lcccccc}
\toprule
$k$ & 1 & 5 & 10 & 20 & 50 & 100 \\
\midrule
Acc (\%) & 43.2\tiny{$\pm$7.8} & 72.7\tiny{$\pm$5.8} & 79.6\tiny{$\pm$1.2} & 82.1\tiny{$\pm$0.7} & \textbf{84.6}\tiny{$\pm$0.5} & 83.2\tiny{$\pm$0.2} \\
\bottomrule
\end{tabular}}
\end{table}

$\mu$-sensitivity: MNIST optimal $\mu{=}10$, CIFAR optimal $\mu{=}1$ (CNNs need more plasticity). The range $\mu \in [1, 10]$ is robust. L2 projection is the natural loss for eigenmode-specific regularization.

\textbf{Measured $k_G$.} We verify the Kronecker rule's input by measuring the effective rank of the (raw) single-output feature Gram $G_{ij}=\phi(x_i)^\top\phi(x_j)$ -- the object that $K_{AA}=I_C\otimes G$ depends on -- using $10$ probes/class. We report both the participation ratio $\mathrm{PR}=(\sum_i\lambda_i)^2/\sum_i\lambda_i^2$ and $k_{50}$ (modes for $50\%$ energy):
\begin{center}
\small
\begin{tabular}{lccc}
\toprule
Features & $\mathrm{PR}$ & $k_{50}$ & top-1 \\
\midrule
TinyMLP penultimate (MNIST)        & 3.1 & 1 & 0.50 \\
Frozen ResNet-18 (CIFAR-10)        & 2.1 & 1 & 0.69 \\
Frozen ResNet-18 (CIFAR-100)       & 2.4 & 1 & 0.65 \\
\bottomrule
\end{tabular}
\end{center}
All three sit in $k_G\in[1,5]$ (a single dominant direction, $k_{50}{=}1$, with $\mathrm{PR}\approx2$--$3$), confirming the input to $k^\star\approx C\cdot k_G$. For CIFAR-100 ($C{=}100$) this gives $k^\star\approx100$--$240$, consistent with the $k{=}100$ matched point at which targeted and broad function-space methods converge (Section~\ref{sec:peft_cl}).

\end{document}